\documentclass[a4paper,12pt,oneside]{article}
\usepackage[margin=1in]{geometry}
\pdfoutput=1

\usepackage{libertine}
\usepackage{courier}
\usepackage{mathpazo}
\usepackage[english]{babel}
\usepackage{times}
\usepackage{amssymb}
\usepackage{natbib}
\usepackage{amsmath}
\usepackage{mathtools}
\usepackage{hyperref}

\usepackage{tikz}
\usepackage{float}
\usepackage{ifthen}

\def\d{{\mathrm d}}
\newcommand{\nrm}[1]{\left\Vert #1 \right\Vert}

\newcommand{\R}{\mathbb{R}}

\newcommand{\identity}{\boldsymbol{I}}
\newcommand{\N}{\mathbb{N}}

\newcommand{\Q}{\boldsymbol{Q}}
\newcommand{\estQ}{\widehat{\boldsymbol{Q}}}

\newcommand{\paren}[1]{\left( #1 \right)}

\newcommand{\tlprn}[1]{\left\{ #1 \right\}}

\newcommand{\abs}[1]{\left| #1 \right|}

\newcommand{\beq}{\begin{eqnarray*}}
\newcommand{\eeq}{\end{eqnarray*}}
\newcommand{\beqn}{\begin{eqnarray}}
\newcommand{\eeqn}{\end{eqnarray}}
\newcommand{\ben}{\begin{enumerate}}
\newcommand{\een}{\end{enumerate}}
\newcommand{\bit}{\begin{itemize}}
\newcommand{\eit}{\end{itemize}}
\newcommand{\hide}[1]{}

\newcommand{\argmin}{\mathop{\mathrm{argmin}}}

\newcommand{\trn}{^\intercal} %
\newcommand{\inv}{^{-1}} %
\newcommand{\unit}{\boldsymbol{1}}
\newcommand{\zero}{\boldsymbol{0}}

\newcommand{\diag}{\operatorname{diag}}

\newcommand{\gn}{\, | \,}

\newcommand{\pr}{\boldsymbol{\mathrm{P}}}

\DeclareMathOperator{\Binomial}{Binomial}
\DeclareMathOperator{\Geometric}{Geometric}

\DeclareMathOperator{\Spec}{Spec}
\newcommand{\Var}[2][]{
		\mathbb{V}\mathbf{ar}_{#1}\left[ #2 \right]
}
\newcommand{\PR}[2][]{
		\mathbb{P}_{#1}\left( #2 \right)
}

\newcommand{\estmc}{\widehat{\mc}} %
\newcommand{\trans}[2]{N_{#1 #2}} %
\newcommand{\etab}{\boldsymbol{\eta}}

\newcommand{\btau}{\boldsymbol{\tau}}

\newcommand{\esttau}{\hat{\tau}}
\newcommand{\estbtau}{\hat{\btau}}
\newcommand{\thit}[1]{U_{#1}}
\newcommand{\tall}{T\cliq}

\newcommand{\mmrisk}{\mathcal{R}_{m}}
\newcommand{\X}{\boldsymbol{X}}

\newcommand{\Y}{\boldsymbol{Y}}

\newcommand{\Gclass}{\mathcal{G}}

\newcommand{\kl}[2]{D_{KL}\left(#1 \middle| \middle| #2\right)}

\newcommand{\eps}{\varepsilon}
\newcommand{\vertiii}[1]{{\left\vert\kern-0.25ex\left\vert\kern-0.25ex\left\vert #1 
    \right\vert\kern-0.25ex\right\vert\kern-0.25ex\right\vert}}

\newcommand{\tv}[1]{\nrm{#1}_{\mathsf{TV}}}
\newcommand{\TV}[1]{\vertiii{#1}}

\renewcommand{\H}{\mathcal{H}}
\newcommand{\M}{\mathcal{M}}
\newcommand{\mc}{\boldsymbol{M}}
\newcommand{\bpi}{\boldsymbol{\pi}}
\newcommand{\bmu}{\boldsymbol{\mu}}

\newcommand{\pred}[1]{\boldsymbol{1}\tlprn{#1}}
\newcommand{\oclprn}[1]{\left( #1 \right]}

\renewcommand{\epsilon}{\eps}
\renewcommand{\pr}{\mathbb{P}}
\newcommand{\prob}[1]{\pr\paren{#1}}
\newcommand{\probm}[2]{\pr_{#1}\paren{#2}}
\newcommand{\E}[2][]{
		\mathbb{E}_{#1}\left[ #2 \right]
}
\newcommand{\expo}[1]{\exp\left( #1 \right)}
\newcommand{\sg}{\gamma}
\newcommand{\pssg}{\gamma_{\mathsf{ps}}}

\newcommand{\clearn}{c}
\newcommand{\cliq}{_{\textrm{{\tiny \textup{CLIQ}}}}}

\newcommand{\cover}{}

\newcommand{\tmix}{t_{\textrm{\textup{mix}}}}
\newcommand{\mnrm}[1]{\nrm{#1}_{\mathsf{MAX}}}

\newcommand{\pimin}{\pi_\star}
\newcommand{\eqdef}{\doteq}
\newcommand{\x}{\boldsymbol{x}}
\newcommand{\y}{\boldsymbol{y}}
\newcommand{\asg}{\gamma_\star}
\newcommand{\rev}{^{\dagger}}
\newcommand{\e}{\boldsymbol{e}}
\newcommand{\law}{\mathcal{L}}
\newcommand{\tildemc}{\widetilde{\mc}}
\newcommand{\tildebpi}{\tilde{\bpi}}

\newcommand{\set}[1]{\left\{ #1 \right\}} %
\newtheorem{theorem}{Theorem}[section]
\newtheorem{lemma}{Lemma}[section]
\newtheorem{definition}{Definition}[section]
\newtheorem{corollary}{Corollary}[section]
\newtheorem{remark}{Remark}[section]
\newenvironment{proof}{\paragraph{Proof:}}{\hfill$\square$}

\numberwithin{equation}{section}
\hypersetup{
	colorlinks=true,
  linkcolor=red,          %
  citecolor=blue,         %
  filecolor=magenta,      %
  urlcolor=cyan           %
}

\begin{document}

\title{Statistical Estimation of Ergodic Markov Chain Kernel over Discrete State Space}

\author{Geoffrey Wolfer \\ \texttt{geo.wolfer@gmail.com} \and Aryeh Kontorovich \\ \texttt{karyeh@cs.bgu.ac.il}}

\maketitle

\begin{abstract}
We investigate the statistical complexity of estimating the parameters of a discrete-state
Markov chain kernel from a single long sequence of state observations.
In the finite case, we characterize
(modulo logarithmic factors)
the minimax
sample complexity of estimation with respect to the operator infinity norm,
while in the countably infinite case, we analyze the problem with respect to a natural entry-wise norm derived from total variation.
We show that in both cases, the sample complexity
is governed by
the mixing properties of the unknown chain,
for which,
in the finite-state case,
there are known finite-sample estimators with fully empirical confidence intervals.
\end{abstract}

\section{Introduction}
\label{sec:intro}
Approximately recovering the parameters of a discrete distribution
is a classical problem in computer science and statistics
(see, e.g., 
\citet{han-minimax-l1-2015,DBLP:conf/colt/KamathOPS15,DBLP:conf/nips/OrlitskyS15}
and the references therein).
Total variation (TV) is a natural and well-motivated choice
of approximation metric
\citep{MR1843146}, and
the two
metrics
we use 
throughout the paper
will be derived from TV.
The minimax sample complexity for obtaining an
$\eps$-approximation
to the unknown distribution
in TV (but see
\citet{DBLP:conf/innovations/Waggoner15} for results on other $\ell_p$ norms) is well-known to be 
of the order of $d/\eps^2$,
where $d$ is the support size
(see, e.g., \citet{MR1741038,KonPin2019}).

This paper deals with estimating the transition probability kernel
of a discrete state time-homogeneous Markov chain
in the minimax setting.
The Markov case
is much less well-understood than the iid one.
The main additional complexity introduced by the Markov case on top of the
iid one
is that
the sample complexity involves not only
the
number of states
and the precision parameter $\eps$,
but also
the chain's mixing properties.

\paragraph{Our contributions.}
In the finite-state case,
we compute, up to logarithmic factors,
(apparently the first, in {\em any} metric)
high-probability minimax sample complexity
for the estimation problem in
the Markovian setting, which
seeks to recover,
from a single long run of an unknown Markov chain,
the values of its transition matrix
up to a tolerance of $\eps$ in the
$\nrm{\cdot}_\infty$ operator norm.
We obtain upper and lower bounds on the sample complexity
(sequence length)  in terms of $\eps$, the number of states,
the stationary distribution,
and mixing time of the Markov chain.

In the countably infinite case,
for a natural class of chains and with respect to an entry-wise metric 
derived from TV, we derive an upper bound on the sample complexity that depends
in a delicate way on some measure of complexity of the kernel, precision $\eps$ 
and mixing time, and provide sufficient conditions on the kernel and initial distribution 
to obtain convergence guarantees.

\section{Definitions and notation}
\label{sec:defnot}

We denote by $\Omega$ the state space of the Markov chain and by $m$
the size of the sample received by the estimation procedure.
The simplex of all distributions over $\Omega$ will be denoted by $\Delta_\Omega$, and the set of all Markov kernels by $\M_\Omega$. 
For $\abs{\Omega} < \infty$, we
put $d:=|\Omega|$
and
$[d] = \Omega = \set{1, 2, \dots , d}$.
For $\bmu \in \Delta_\Omega$, we will write either $\bmu(i)$ or $\mu_i$,
as dictated by esthetics and convenience.
All vectors are rows unless indicated otherwise. We use the standard total variation norm,
which, up to a convention-dependent factor of $2$,
corresponds to the $\ell_1$ norm: $2 \tv{\x} = \nrm{\x}_1 = \sum_{i\in \Omega}|x_i|$. 
We assume familiarity with basic Markov chain concepts
(see, e.g., \citet{MR0410929,levin2009markov}).
A time-homogeneous Markov chain $(\mc,\bmu)$ on state space $\Omega$ is specified
by an initial distribution $\bmu\in\Delta_\Omega$
and a kernel $\mc \in \M_\Omega$ in the usual way:
$(X_1,\ldots,X_m)\sim(\mc,\bmu)$ means that
$$\prob{(X_1,\ldots,X_m)=(x_1,\ldots,x_m)}=\bmu(x_1)\prod_{t=1}^{m-1}\mc(x_t,x_{t+1}).$$
We write $\probm{\mc,\mu}{\cdot}$ to denote
probabilities
over sequences
induced by 
the Markov chain
$(\mc,\bmu)$,
and omit one or both subscripts when clear from context. We say that $\bpi$ is a \emph{stationary distribution} for $\mc$ if $\bpi \mc = \bpi$, and that
the Markov chain $(\mc,\bmu)$ is {\em stationary} if $\bmu=\bpi$. We will assume the chain to be \emph{irreducible} and \emph{positive recurrent}.
Namely,
$\mc$ consists of a single communicating class,
and defining the \emph{return time} of state $i$ as $T_i = \min \set{t \geq 0 : X_t = i}$, 
we have that for any state $i \in \Omega, \E{T_i} < \infty$. This is sufficient to guarantee existence of a stationary $\bpi$. 
We will further restrict our analysis to \emph{geometrically ergodic} Markov chains to enable spectral methods.
\begin{definition}[Geometric ergodicity, \cite{roberts1997geometric}]
\label{definition:geometric-ergodicity}
The chain $(\mc,\bmu)$ with stationary distribution $\bpi$ is geometrically ergodic if
there is a 
$\rho \in (0,1)$ and
for all
$i \in \Omega$
there is a
$C_i \in \mathbb{R}_+$ such that
$$\tv{\mc^t(i, \cdot) - \bpi} \leq C_i \rho^t, \qquad t \in \mathbb{N}.$$
\end{definition}
\noindent Any chain that satisfies all the above properties will henceforth simply be called \emph{ergodic}, and all chains mentioned in this work will be assumed ergodic unless stated otherwise.
If $\mc$ is ergodic
with stationary distribution
$\bpi$,
then $\bpi$ is necessarily unique.
To any Markov chain $(\mc,\bmu)$, we associate the following measure of non-stationarity
\beqn
\label{eq:Pimu-def}
\nrm{\bmu/\bpi}_{2,\bpi}^2 \eqdef \sum_{i \in \Omega} \bmu(i)^2/\bpi(i) \in
[1,\infty]
,
\eeqn
where the
$\nrm{\cdot}_{2,\bpi}$
norm
is
induced by the inner product in the Hilbert space
$\ell_2(\bpi)$ \citep[Chapter~12]{levin2009markov}.
When $\abs{\Omega} < \infty$, we can define the \emph{minimum stationary probability} by
\beqn
\label{eq:pistar-def}
\pimin \eqdef \min_{i\in \Omega }\bpi(i).
\eeqn
In this case, by ergodicity $\pimin > 0$ and $\nrm{\bmu/\bpi}_{2,\bpi}^2 \leq \frac{1}{\pimin} < \infty$. The \emph{mixing time} of an ergodic $\mc$ is defined by
\beqn
\label{eq:mixing-time}
\tmix \eqdef \inf \set{ t \geq 1 : \sup_{\bmu \in \Delta_\Omega} \tv{\bmu \mc^t - \bpi} < \frac{1}{4}}.
\eeqn
\noindent
We define $\Q \eqdef \diag \left( \bpi \right) \mc$
as
the matrix
$\Q(i,j) = \PR[\bpi]{X_{t} = i, X_{t+1} = j}$.
A chain $\mc \in \M_\Omega$  is said to be {\em reversible} if $\Q \trn = \Q$. The eigenvalues of an ergodic and reversible $\mc$ lie in $\oclprn{-1,1}$,
and thus may be ordered (counting multiplicities): $1 = \lambda_1\ > \lambda_2\ge\ldots\ge\lambda_d > -1$. The {\em spectral gap} and \emph{absolute spectral gap} of a reversible chain
are defined, respectively, by
\beqn
\label{eq:gamma-def}
\sg \eqdef 1-\lambda_2 \text{ and } \asg \eqdef 1 - \max \set{ \lambda_2, \abs{\lambda_d}}.
\eeqn
\citet{paulin2015concentration} generalizes the
      {\em multiplicative reversiblization} approach of \citet{0726.60069}
      by defining
the {\em pseudo-spectral gap}
\beqn
\label{eq:pseudo-gamma-def}
\pssg \eqdef \max_{k \geq 1} \set{ \sg((\mc \rev)^k \mc^k)/k },
\eeqn
where $\mc \rev$ is the {\em time reversal} of $\mc$
--- the adjoint of $\mc$ under $\ell_2(\bpi)$
--- given by $\mc \rev(i,j) \eqdef \bpi(j)\mc(j,i)/\bpi(i)$. 

For a linear operator $\boldsymbol{A}: \Omega \to \Omega$,
\beqn
\label{eq:normdef}
\nrm{\boldsymbol{A}}_\infty = \sup_{i\in \Omega}\sum_{j \in \Omega}\abs{\boldsymbol{A}(i,j)}
\eeqn
is the operator norm induced by
$\ell_\infty$
\citep{horn-johnson}.
We also define
the following entry-wise norm
\beqn
\label{eq:normtvdef}
\TV{\boldsymbol{A}} \eqdef \sum_{(i,j) \in \Omega^2}\abs{\boldsymbol{A}(i,j)}.
\eeqn
The norms in
\eqref{eq:normdef} and \eqref{eq:normtvdef}
induce our two notions of distance between
Markov kernels $\mc, \mc'$ with respective stationary distributions $\bpi$ and $\bpi'$:
$$\nrm{\mc - \mc'}_\infty \text{ and } \TV{\boldsymbol{Q} - \boldsymbol{Q}'}.$$
For any $\mc\in\M_\Omega$, define its {\em Dobrushin contraction} coefficient
\beqn
\label{eq:dobr}
\kappa(\mc) \eqdef
\max_{(i,j) \in \Omega^2} \tv{ \mc(i, \cdot) - \mc(j, \cdot)};
\eeqn
this quantity is also associated with D\"oblin's name. The term ``contraction''
refers to the property
\beqn
\label{eq:mark-conc}
\tv{
  (\bmu-\bmu')\mc
}
\le\kappa(\mc)\tv{\bmu-\bmu'}
,
\qquad
(\bmu,\bmu') \in \Delta_\Omega^2,
\eeqn
which was observed
by
\citet[$\mathsection$5]{mar1906}
(see \citet[Lemma A.2]{kontram06} for an elementary proof).

\section{Main results}
\label{sec:main-res}
In Section~\ref{section:results-finite-case} we formally state the minimax results for the finite state setting, 
and then exhibit our results for the countably infinite case in Section~\ref{section:results-infinite-case}.

\subsection{Estimation with respect to \texorpdfstring{$\nrm{\cdot}_\infty$}{the infinity norm}
  for finite \texorpdfstring{$\Omega$}{state space}
}
\label{section:results-finite-case}

\begin{theorem}[Sample complexity upper bound w.r.t $\nrm{\cdot}_\infty$ when $\abs{\Omega} < \infty$]
  \label{thm:learn-ub}
  Let $\eps \in (0, 2)$, $\delta \in (0, 1)$, and let $\X=(X_1,\ldots,X_m) \sim (\mc, \bmu)$, $\mc$ ergodic with stationary distribution $\bpi$.
Then an  
estimator $\estmc: \Omega^m \to \M_\Omega$ exists such that
whenever
$$m \geq \clearn \max \set{ \frac{1}{\eps^2 \pimin} \max \set{ d, \ln \frac{1}{\eps \delta} }, \frac{1}{\pssg \pimin} \ln \frac{d \nrm{\bmu/\bpi}_{2, \bpi}  }{\delta} }$$
we have,
with probability at least $1 - \delta$,
\beq
\nrm{\mc - \estmc}_\infty < \eps,
\eeq
where
$\clearn$ is a universal constant,
$d = \abs{\Omega}$,  $\pssg$ is the \emph{pseudo-spectral gap}
(\ref{eq:pseudo-gamma-def}), $\pimin$ the minimum stationary probability (\ref{eq:pistar-def}),
and $\nrm{\bmu/\bpi}_{2,\bpi}^2 \le1/\pimin$
is defined in (\ref{eq:Pimu-def}).
\end{theorem}

Although the sample complexity depends on the spectral quantity $\pssg$,
and minimal stationary probability $\pimin$ of the unknown chain,
these can be efficiently estimated with finite-sample data-dependent confidence intervals
from a single trajectory \citep{hsu2019, pmlr-v99-wolfer19a}. 
Moreover, even though the upper bound formally depends on the unknown
(and, in our one-trajectory setting, impossible to estimate)
initial distribution $\bmu$, we note that $(i)$ this dependence is only logarithmic
and $(ii)$ an upper bound on $\nrm{\bmu/\bpi}_{2,\bpi}^2$ in terms of $\pimin$ is easily provided.

\begin{remark}
This upper bound is superior to the one given at \citet[Theorem~1]{pmlr-v98-wolfer19a},
shaving a multiplicative factor of $\ln{d}$ off the first term,
except in the extremely high precision regime where $\ln \frac{1}{\eps} \geq d \ln{d}$.
\end{remark}

\begin{theorem}[Sample complexity lower bound w.r.t $\nrm{\cdot}_\infty$ when $\abs{\Omega} < \infty$]
  \label{thm:learn-lb}
~For every $\eps \in (0, 1/32)$, $\pssg \in (0, 1/8)$,
$d=6k\ge12$,
and every estimation procedure,
there exists
a $d$-state Markov chain $\mc$
with
pseudo-spectral gap $\pssg$ and stationary distribution $\bpi$ such that
the
estimation procedure
must require
a sequence $\X=(X_1,\ldots,X_m)$
drawn from the unknown $\mc$
of length at least
\beq
m \geq c \max\set{ \frac{d}{\eps^2 \pimin}, \frac{d \ln d}{\pssg}
} ,
\eeq
where $c$ is a universal constant,
to ensure $\nrm{\mc - \estmc}_\infty < \eps$ with probability greater than $9/10$,
and where $d, \pssg, \pimin$ are as in Theorem~\ref{thm:learn-ub}.
\end{theorem}

The proof of Theorem~\ref{thm:learn-lb}
actually yields a bit more than claimed in the statement.
For any
$\pimin \in \oclprn{0, 1/d}$,
a Markov chain $\mc$ can be constructed that achieves the
$\frac{d}{\eps^2\pimin}$ component of the bound.
Additionally,
the
$\frac{d}{\pssg}$ component
is achievable by a class of {\em reversible} Markov chains
with spectral gap $\sg \leq \pssg \leq 2 \sg$, and uniform stationary 
distribution --- for which $\pimin = 1/d$ --- exhibiting tightness of the obtained bound.

The form of the lower bound indicates that in some regimes,
estimating the pseudo-spectral gap up to constant multiplicative error, 
which requires $\frac{d}{\pssg}$ \citep{hsu2019, pmlr-v99-wolfer19a},
is as difficult as estimating the entire transition matrix
(for our choice of metric $\nrm{\cdot}_\infty$).
We stress that our procedure and guarantees only require ergodicity
(and not, say, reversibility) to work.

\subsection{Results for estimation with respect to \texorpdfstring{$\TV{\cdot}$}{|||.|||}}
\label{section:results-infinite-case}

Over an infinite space, $\pimin=0$
conveys no information,
which motivates an alternative notion of distance. 
For a chain $\mc$, the kernel of \emph{doublet frequencies} $\Q = \diag(\bpi) \mc$ encodes all 
information about an ergodic chain \citep{vidyasagar2014elementary}, and for two such operators, 
$\Q$ and $\Q'$ it is the case that 
$$\Q = \Q' \implies \mc = \mc'.$$
Further, it is easily verified
that 
$$\TV{\Q - \Q'} = 2 \tv{\Q - \Q'}
,$$
where we see $\Q$ and $\Q'$ as distributions over $\Omega \times \Omega$.

\begin{remark}
  The loss of our estimation problem is
  distinct from the one considered in \citet{haoorlitsky2018}, 
  which weights the state-wise expected loss with respect to the stationary
  distribution of the chain,
  and also allows for sample bounds independent of $\pimin$.
\end{remark}

\begin{theorem}
\label{theorem:instance-specific-upper-bound-finite}
Let $\eps \in (0, 2), \delta \in (0, 1)$, and $\X = (X_1, \dots, X_m) \sim (\mc, \bmu)$, $\mc$ ergodic 
with stationary distribution $\bpi$, and write $\boldsymbol{Q} = \diag \left( \bpi \right) \mc$. 
There exists an estimator $\widehat{\boldsymbol{Q}}: \Omega^m \to \Delta_{\Omega \times \Omega}$ such that for
\begin{equation*}
m \geq c \frac{\tmix}{\eps^2} \max \set{ \TV{\boldsymbol{Q}}_{1/2} , \ln \left( \frac{\nrm{\bmu/\bpi}_{2, \bpi}}{\delta} \right) },
\end{equation*}
we have
$\TV{\widehat{\boldsymbol{Q}} - \boldsymbol{Q}} < \eps$ with probability at least $1- \delta$, where
\begin{equation*}
\TV{\boldsymbol{Q}}_{1/2} \eqdef \left( \sum_{(i,j) \in \Omega^2} \sqrt{\Q(i,j) } \right)^2,
\end{equation*}
$c$ is a universal constant, $\tmix$ is defined at \eqref{eq:mixing-time}, and $\nrm{\bmu/\bpi}_{2, \bpi}$ is defined in \eqref{eq:Pimu-def}.
\end{theorem}

\begin{remark}
Necessary conditions for the  upper bound to be non-vacuous are that both $\TV{\boldsymbol{Q}}_{1/2} < \infty$ and $\nrm{\bmu/\bpi}_{2, \bpi} < \infty$.
Importantly, $\nrm{\bmu/\bpi}_{2, \bpi} < \infty$ implies but is not implied by $\bmu \ll \bpi$;
take, e.g.,
$\bmu(i) \propto \frac{1}{i^2}$ and $ \bpi(i) \propto \frac{1}{i^4}$.
Notice that in the special case where $\abs{\Omega} = d < \infty$,
we have
$\TV{\Q}_{1/2} \leq d^2$ and the bound reduces to
$\tmix \frac{ d^2}{\eps^2}$ (up to logarithmic factors).
The mixing time $\tmix$, unknown a priori, can be estimated 
with finite-sample empirical intervals \citep{pmlr-v117-wolfer20a}.
\end{remark}

\section{Overview of techniques}
\label{sec:techniques}
\subsection{Estimating with respect to the \texorpdfstring{$\nrm{\cdot}_\infty$}{infinity} norm when \texorpdfstring{$\abs{\Omega} < \infty$}{the space is finite}}

The upper bound for the estimation problem in Theorem~\ref{thm:learn-ub}
is achieved by
a (mildly smoothed) natural estimator defined at 
the beginning of Section~\ref{sec:proof-ub-finite}.
If
the stationary distribution
is
bounded away from $0$,
the chain will
visit each state a constant fraction
of the total sequence length.
Exponential concentration (controlled by the spectral gap)
provides high-probability confidence intervals about the expectations.
A technical complication is that the empirical distribution
of the transitions out of a state $i$, conditional on the number of visits
$N_i$
to that state, is not binomial but actually rather complicated
--- this is due to the fact that the sequence length is fixed
and so a large value of $N_i$ ``crowds out'' other observations.
We overcome this by
simulating a trajectory from the Markov chain with 
an array of independent random variables, as described in \citet[p.19]{Billingsley61}.
The factor $\nrm{\bmu/\bpi}_{2, \bpi}$ in the bounds quantifies
the price one pays for not assuming (as we do not)
stationarity
of the unknown Markov chain.

Our chief technical contribution is in establishing
the sample complexity
lower bounds for the finite space estimation problem.
We do this by constructing two distinct lower bounds. 

The
lower bound of $\frac{d \ln{d}}{\pssg}$ is derived 
by a covering argument
and
a classical reduction scheme to a collection of testing problems
using
a class of reversible Markov chains
we
construct,
with
a carefully controlled pseudo-spectral gap.\footnote{The family of chains
used in the lower bound of \citet{hsu2019} does not suffice
for our purposes; a considerably richer family is needed (see Remark~\ref{rem:HSK}).
}
The latter can be upper and lower bounded up to universal constants in three key steps.
First, we leverage the block structure of the transition matrix of
the non-perturbed member of the family to compute its entire spectrum explicitly (Lemma~\ref{lemma:exact-eigenvalues}),
and deduce its absolute spectral gap.
We then extend the bound to other members of the family, 
using \emph{Markov chain comparison techniques}, 
going through a well known variational definition of the spectral gap.
Finally, we conclude by showing that the pseudo-spectral and spectral gap
are within a factor of 2 for our class of symmetric Markov chains.

The lower bound of $\frac{d}{\eps^2 \pimin}$
is based on the observation that estimating the whole kernel is at least
as hard as estimating the conditional distribution a single state.
From here, we construct a class of matrices where one state is
both
hard to reach
and difficult to estimate,
by
constructing
mixture of
indistinguishable distributions for that particular state,
indexed by a large subset of the binary hypercube.
We
express the statistical distance between words of length $m$
distributed according to different matrices of this class in
terms of $\pimin$ and the KL divergence between the
conditional distributions of the hard-to-reach state,
by taking advantage of the structure of the class,
and invoke an argument from Tsybakov to conclude ours.

\subsection{Estimating with respect to the \texorpdfstring{$\TV{\cdot}$}{|||.|||} metric}

The extension to a countably infinite
setting requires an alternative notion of distance between chains.
The proof then introduces the natural counting estimator $\estQ(i,j) = \frac{N_{i j}}{m-1}$ of transitions from $i$ to $j$,
and starts by controlling the error in expectation. It reduces the problem to the study of the variance of the random variable $N_{i j}$, 
which is achieved by constructing another Markov chain with approximately the same mixing time, and invoking known 
results from \citet{paulin2015concentration} for the variance of sums of functions under the Markovian setting. 
The result is then obtained by controlling the fluctuations around this expectation by a bounded differences argument.

\section{Related work}
\label{sec:related-work}
Our Markov chain statistical estimation setup is a natural extension
of the PAC distribution learning model of \citet{DBLP:conf/stoc/KearnsMRRSS94}.
Despite the plethora of literature on estimating Markov transition matrices,
(see, e.g., \citet{Billingsley61,craig2002estimation,welton2005estimation})
we were not able to locate
any rigorous finite-sample PAC-type results. 

The minimax problem has recently received some attention,
and \citet{haoorlitsky2018} have, in parallel to us, shown the first minimax bounds,
in expectation,
for the problem of estimating the transition matrix $\mc$ of a Markov chain
under a certain class of divergences.
The authors consider the case where
$\min_{i,j} \mc(i, j)\ge\alpha > 0$,
essentially showing that for
some family of smooth $f$-divergences,
the expected risk is of the order of
$\frac{d f''(1)}{m \pimin}$.
The metric used in this paper is based on TV, which corresponds
to
the $f$-divergence
induced by
$f(t) = \frac12 \abs{t-1}$, which is not differentiable at $t = 1$.
The results of \citeauthor{haoorlitsky2018}
and the present paper
are complementary and not directly comparable.
We do note that
$(i)$ their guarantees are in expectation
rather than with high-confidence,
$(ii)$ our TV-based metric is not covered by their
smooth
$f$-divergence family,
and most important
$(iii)$ their notion of mixing is related to contraction
as opposed to the spectral gap.
In particular the $\alpha$-minoration assumption implies (but is not implied by)
a bound of $\kappa\le1-d\alpha$ on the Dobrushin contraction coefficient (defined in (\ref{eq:dobr});
see \citet[Lemma 2.2.2]{kont07-thesis} for the latter claim).
Thus, the family of $\alpha$-minorized Markov chains is strictly contained in the family of
contracting chains, which in turn is a strict subset of the ergodic chains we consider.

This paper is based on the conference version of \citet{pmlr-v98-wolfer19a} 
together with an extension to countably infinite spaces at Section~\ref{section:results-infinite-case}. 
Another key improvement over the extended abstract 
is in the proof of Theorem~\ref{thm:learn-lb}.
While the series of lemmas \citet[Lemma~8, Lemma~9, Lemma~11]{pmlr-v98-wolfer19a} 
showed that it is possible to control the pseudo-spectral
gap of our special family of chains via Cheeger's inequality
combined with a contraction-based argument, this technique relied
on heavy computations to bound the Dobrushin coefficient of the
two-step transition matrix. Moreover the proof for the extension to all members 
of the class was only sketched in \citet[Lemma~9]{pmlr-v98-wolfer19a}. 
In the present manuscript, we switch technique, 
compute the full spectrum of the unperturbed transition matrix instead,
and fully flesh out the proof for the extension to perturbed chains
using comparison techniques.
Finally, the upper bound at Theorem~\ref{thm:learn-ub} 
also improves upon
\citet[Theorem~1]{pmlr-v98-wolfer19a},
by relying on a simulation scheme from Billingsley,
instead of martingale techniques.

\section{Proofs}
\label{sec:proofs}

\subsection{Proof of Theorem~\ref{thm:learn-ub}}
\label{sec:proof-ub-finite}
\begin{remark}
  We thank an anonymous referee for the suggestion (and technique)
  to improve the logarithmic gap between the upper and lower bounds.
\end{remark}

Let $\eps \in (0, 2), \delta \in (0, 1)$, let $\mc$ be a $d$-state ergodic Markov kernel with stationary distribution $\bpi$, 
and first consider the stationary case $X_1, \dots, X_m \sim (\mc, \bpi)$. We define the natural counting random variables
\begin{equation*}
\begin{split}
N_i \eqdef \sum_{t = 1}^{m-1} \pred{X_t = i}, \qquad \trans{i}{j} \eqdef \sum_{t = 1}^{m-1} \pred{X_t = i, X_{t+1} = j}
\end{split}
\end{equation*} 
and the estimator of the kernel will be $\estmc(i, j) \eqdef \cfrac{\trans{i}{j}}{N_i}$ when $N_i \neq 0$ and $1/d$ when $N_i = 0$.
We decompose the error probability of the estimation procedure, 
while choosing an arbitrary value $n_i \in \N$ for the desired number of visits to each state $i \in [d]$, 
\begin{equation}
\label{eq:decomposition}
\begin{split}
\PR[\bpi]{\nrm{ \mc - \estmc }_\infty > \eps} &\leq \sum_{i=1}^{d} \sum_{n = n_i}^{3 n_i} \PR[\bpi]{ \nrm{\estmc(i, \cdot) - \mc(i, \cdot)}_1 > \eps \text{ and } N_i = n } \\
& + \PR[\bpi]{ \set{\exists i \in [d]: N_i \notin [n_i, 3n_i]} }
.
\end{split}
\end{equation}
We simulate a trajectory from $\mc$ with a collection of independent 
samples using the scheme described in \citep[p.19]{Billingsley61},
where we define the following infinite array of random variables,
\begin{equation*}
\begin{matrix}
X_{1,1} & X_{1,2} & \cdots & X_{1,t} & \cdots \\
X_{2,1} & X_{2,2} & \cdots & X_{2,t} & \cdots \\ 
\cdots & \cdots & \cdots & \cdots & \cdots \\ 
X_{d,1} & X_{d,2} & \cdots & X_{d,t} & \cdots \\ 
\end{matrix}
\end{equation*}
such that $\forall (i,j,t) \in [d]^2 \times \N, \PR{X_{i, t} = j} = \mc(i,j)$,
and the sampling procedure is as follows.
Start by drawing $\tilde{X}_1 \sim \bpi$.
$\tilde{X}_2$ is then defined to be $X_{\tilde{X}_1, 1}$, 
the first element of the $\tilde{X}_1$th row.
The process then continues inductively
recording random variables 
from left to right in their corresponding rows,
such that if $\tilde{X}_1, \tilde{X}_2, \dots, \tilde{X}_{t}$ 
have been defined,
then $\tilde{X}_{t+1} \eqdef X_{\tilde{X}_t, \tilde{N}_{\tilde{X}_t}^{(t)} + 1}$ 
where $\tilde{N}_{i}^{(t)} \eqdef \sum_{s = 1}^{t} \pred{\tilde{X}_s = i}$,
and for convenience, $\tilde{N}_{i} \eqdef \tilde{N}_{i}^{(m-1)}$.
Observe that $X_1, X_2,\dots,$ and $\tilde{X}_1, \tilde{X}_2, \dots$ 
are identically distributed. Then, writing
\begin{equation*}
\begin{split}
\widetilde{\mc}(i, \cdot) \eqdef \frac{1}{\tilde{N}_{i}} \sum_{j = 1}^{d} \sum_{t = 1}^{m-1} \pred{\tilde{X}_t = i, \tilde{X}_{t+1} = j} \e_j, 
\end{split}
\end{equation*}
we have
\begin{equation*}
\begin{split}
\PR[\bpi]{ \nrm{\estmc(i, \cdot) - \mc(i, \cdot)}_1 > \eps \text{ and } N_i = n } &= \PR{ \nrm{ \widetilde{\mc}(i, \cdot) - \mc(i, \cdot)}_1 > \eps \text{ and } \tilde{N}_i = n}. \\
\end{split}
\end{equation*}
In the event where $\tilde{N}_i = n$,
\begin{equation*}
\begin{split}
\widetilde{\mc}(i, \cdot) &= \frac{1}{n} \sum_{j = 1}^{d} \sum_{t = 1}^{m-1} \pred{\tilde{X}_t = i, X_{\tilde{X}_t, \tilde{N}_{\tilde{X}_t}^{(t)} + 1} = j} \e_j \\
&= \frac{1}{n} \sum_{j = 1}^{d} \sum_{t = 1}^{n} \pred{X_{i, t} = j} \e_j \\
&\eqdef \widetilde{\mc}_n(i, \cdot),
\end{split}
\end{equation*}
where by definition, $X_{i,1}, X_{i, 2}, \dots, X_{i, n} \sim \mc(i, \cdot)^{\otimes n}$,
and the problem is reduced to learning a distribution out of $n$ independent samples.
Since $\mathbb{E}\nrm{\widetilde{\mc}_n(i, \cdot) - \mc(i, \cdot)}_1 \leq \sqrt{\frac{d}{n}}$, 
(see for example \citet{Berend2013})
and the function $X_{i,1}, X_{i, 2}, \dots, X_{i, n} \mapsto \nrm{\widetilde{\mc}_n(i, \cdot) - \mc(i, \cdot)}_1$ is $(2/n)$-Lipschitz,
an application of McDiarmid's inequality yields that
\begin{equation*}
\begin{split}
\PR{ \nrm{ \widetilde{\mc}(i, \cdot) - \mc(i, \cdot)}_1 > \eps \text{ and } \tilde{N}_i = n} &\leq \PR{ \nrm{ \widetilde{\mc}_n(i, \cdot) - \mc(i, \cdot)}_1 > \eps} \\
&\leq\expo{ -\frac{n}{2} \max \set{ 0, \eps - \sqrt{\frac{d}{n}} }^2 }. \\
\end{split}
\end{equation*}
It follows that
\begin{equation*}
\begin{split}
&\sum_{n = n_i}^{3 n_i}\PR[\bpi]{ \nrm{\estmc(i, \cdot) - \mc(i, \cdot)}_1 > \eps \text{ and } N_i = n } \\
&\stackrel{(i)}{\leq} (2n_i + 1) \expo{ -\frac{n_i}{2} \max \set{ 0, \eps - \sqrt{\frac{d}{n_i}} }^2 } \\
&\stackrel{(ii)}{\leq} (m \pi_i + 1) \expo{ - c m \pi_i \eps^2  }, \\
\end{split}
\end{equation*}
where $c = \frac{(1 - 1/\sqrt{2})^2}{4}$, $(i)$ stems from a monotonicity argument, 
and $(ii)$ is by setting $n_i = \frac{m \pi_i}{2}$, and as long as $m \geq \frac{4d}{ \eps^2 \pimin}$.
We start by handling the first term of \eqref{eq:decomposition},
\begin{equation}
\label{eq:remaining-sum}
\begin{split}
\sum_{i=1}^{d} (m \pi_i + 1) \expo{ - c m \pi_i \eps^2  } &\leq \sum_{i=1}^{d} \frac{1}{c \eps^2} \expo{ - c m \pi_i \eps^2 / 2 } + \sum_{i=1}^{d}  \expo{ - c m \pi_i \eps^2  } \\
&\leq \sum_{i=1}^{d} \frac{2}{c \eps^2} \expo{ - c m \pi_i \eps^2 / 2 } \\
&\leq \frac{2 d}{c \eps^2} \expo{ - c m \pimin \eps^2 / 2 }, \\
\end{split}
\end{equation}
where we used the fact that $x > 0 \implies x \expo{ -x }\leq \expo{-x/2}$, and
which is smaller than $\delta/2$ as long as $m \geq \frac{2}{c \pimin \eps^2} \ln \frac{4d}{c \delta \eps^2}$.
It remains to control the probability of the bad event where the states are not visited a reasonable amount of time. Invoking \citet[Theorem~3.10]{paulin2015concentration},
\begin{equation}
\label{eq:control-visits}
\PR[\bpi]{N_i \notin \left[\frac{1}{2}m \pi_i, \frac{3}{2}m \pi_i\right]} \leq \expo{ -\frac{\pssg \left( \frac{1}{2}m \pi_i \right)^2}{8 (m + 1/\pssg ) \pi_i(1 - \pi_i) + 20 \frac{1}{2}m \pi_i  } }.
\end{equation}
Quantifying the price for non-stationarity using \citet[Proposition~3.14]{paulin2015concentration},
\begin{equation*}
\begin{split}
\PR[\bmu]{\nrm{\estmc - \mc}_\infty > \eps } \leq \nrm{\bmu/\bpi}_{2, \bpi} \sqrt{ \PR[\bpi]{\nrm{\estmc - \mc}_\infty > \eps }},
\end{split}
\end{equation*}
and combining with \eqref{eq:control-visits} yields the upper bound. \hfill\ensuremath{\square}

\begin{remark}
Note that one can derive an upper bound of $\frac{1}{\pimin}\max \set{1/\eps^2,1/\pssg }$ (up to logarithmic factors) for the problem with respect to the max norm 
$$\mnrm{\mc - \estmc} = \max_{(i,j) \in [d]^2} \abs{\mc(i,j) - \estmc(i,j)}.$$ 
Similarly, for $p \in [1, 2)$, we can derive the more general upper bound (up to logarithmic factors)
$$\frac{1}{\pimin} \max \set{ \frac{d^{2/p - 1}}{\eps^2}, \frac{1}{\pssg} }$$ 
for the problem with respect to the norm $\nrm{\mc - \estmc}_{\infty, p} \eqdef \max_{i \in [d]} \nrm{\mc(i, \cdot) - \estmc(i, \cdot)}_p$.
\end{remark}

\subsection{Proof of Theorem~\ref{theorem:instance-specific-upper-bound-finite}}
Sample $X_1, \dots, X_m \sim (\mc, \bmu)$ where $\mc$ is ergodic with stationary distribution $\bpi$.
We define the estimator $\estQ \eqdef \frac{1}{m-1} N_{i j}$ with $N_{i j} \eqdef \sum_{t=1}^{m-1} \pred{X_t = i, X_{t+1} = j}$. 
We first focus on the stationary case where $\bmu = \bpi$.

\paragraph{Bounding the distance in expectation.}

From Jensen's inequality and stationarity,
\begin{equation*}
\begin{split}
\mathbb{E}_{\bpi}\TV{\estQ - \Q} &= \sum_{(i, j) \in \Omega^2}  \mathbb{E}_{\bpi} \abs{\frac{N_{i j}}{m-1} - \bpi(i) \mc(i,j)} \\
&\leq \frac{1}{m-1} \sum_{(i, j) \in \Omega^2}  \sqrt{ \E[\bpi]{\left(N_{i j} - (m-1)\bpi(i) \mc(i,j)\right)^2} }\\
&= \frac{1}{m-1} \sum_{(i, j) \in \Omega^2} \sqrt{ \Var[\bpi]{N_{i j}}},\\
\end{split}
\end{equation*}
and we are left with controlling a variance term. The next lemma defines a new Markov chain from $X_1, \dots, X_m$ with an approximately similar mixing time.
\begin{lemma}
\label{lemma:tuple-chain}
Let $\X = (X_1, \dots, X_m) \sim (\mc, \bpi)$ with mixing time $\tmix$ and stationary distribution $\bpi$, then $\Y = ((X_1, X_2), (X_2, X_3), \dots, (X_{m-1}, X_m) )$ is also a finite state Markov chain with mixing time at most $\tmix + 1$, with kernel $\tildemc: \Omega^2 \times \Omega^2 \to [0, 1]$ and stationary distribution $\tildebpi$, such that for all $ (i,j,k,\ell) \in \Omega^4$,
\begin{equation*}
\begin{split}
\tildemc((i, j), (k, \ell)) &= \pred{k = j} \mc(k, \ell) \\
\tildebpi((i,j)) &= \Q(i,j).
\end{split}
\end{equation*}
\end{lemma}
\begin{proof}
Let $\X = (X_1, \dots, X_m) \sim (\mc, \bpi)$ with mixing time $\tmix$ and stationary distribution $\bpi$, we first show that $\Y = ((X_1, X_2), (X_2, X_3), \dots, (X_{m-1}, X_m) )$ is also a finite state Markov chain. For all $t \in \N$ and $\y = (y_1, \dots, y_{t-1})$ with $y_{s} = (x_{s}, x_{s+1})$ for $s \in [t-2]$, and whenever defined,
\begin{equation*}
\begin{split}
&\PR{Y_t = (i, j) \gn Y_1 = y_1, \dots, Y_{t-1} = y_{t-1}} \\
&= \PR{(X_{t}, X_{t+1}) = (i, j) \gn X_1 = x_1, \dots, X_{t-1} = x_{t-1}, X_{t} = x_{t}} \\
&= \pred{i = x_t} \PR{ X_{t+1} = j \gn X_{t} = x_{t}} = \pred{i = x_t} \mc(i, j) \\
&= \PR{Y_t = (i, j) \gn Y_{t-1} = y_{t-1}}, \\
\end{split}
\end{equation*}
which confirms the Markov property. Additionally, setting $\tildebpi((i,j)) = \Q(i,j)$,
\begin{equation*}
\begin{split}
\sum_{(i,j) \in \Omega^2} \tildebpi((i,j)) \tildemc((i, j), (k, \ell)) &= \sum_{(i,j) \in \Omega^2} \tildebpi((i,j)) \pred{k = j} \mc(k, \ell) = \mc(k, \ell) \sum_{i \in \Omega} \tildebpi((i,k)) \\
& = \mc(k, \ell) \sum_{i  \in \Omega} \bpi(i) \mc(i,k) = \mc(k, \ell) 
\bpi(k) = \tildebpi((k,\ell)).
\end{split}
\end{equation*}
This entails that $\Y$ is a Markov chain $(\tildebpi, \tildemc)$ over the state space $\Omega \times \Omega$, and stationary distribution $\tildebpi$. Let $t \geq \tmix + 1$, and $\tilde{\delta}_{(i_1, j_1)}$ the distribution on $\Omega \times \Omega$ that puts mass 1 at $(i_1, j_1)$ and $0$ everywhere else, then
\begin{equation*}
\begin{split}
\tv{\tilde \delta_{(i_1, j_1)} \tildemc^{t} - \tildebpi} &= \sum_{(i, j) \in \Omega^2} \abs{\tilde \delta_{(i_1, j_1)} \tildemc^{t}(i,j) - \tildebpi(i,j)} .
\end{split}
\end{equation*}
One one hand,
\begin{equation*}
\begin{split}
\tilde \delta_{(i_1, j_1)} \tildemc^{t}(i,j) &= \PR{Y_{t+1} = (i,j) \gn Y_1 = (i_1, j_1)} \\
&= \PR{(X_{t+1}, X_{t+2}) = (i,j) \gn (X_1, X_2) = (i_1, j_1)} \\
&= \PR{X_{t+2} = j \gn X_{t+1} = i} \PR{X_{t+1} = i \gn X_2 = j_1} \\
&= \mc(i,j) \mc^{t-1}(j_1, i), \\
\end{split}
\end{equation*}
so that
\begin{equation*}
\begin{split}
\tv{\tilde \delta_{(i_1, j_1)} \tildemc^{t} - \tildebpi} &= \frac{1}{2} \sum_{(i, j) \in \Omega^2} \abs{\mc(i,j) \mc^{t-1}(j_1, i) - \mc(i,j) \bpi(i)} \\
&= \frac{1}{2} \sum_{i \in \Omega} \abs{\mc^{t-1}(j_1, i) - \bpi(i)} = \tv{ \delta_{j_1} \mc^{t-1} - \bpi} \leq 1/4, \\
\end{split}
\end{equation*}
by definition of $\tmix$, and by the condition on $t$.
\end{proof}

\begin{corollary}
\label{corollary:control-variance-empirical-transition-counts}
Let $X_1, \dots, X_m \sim (\mc, \bpi)$ and define the chain $\Y \sim (\tildemc, \tildebpi)$ from Lemma~\ref{lemma:tuple-chain}, then by \citet[Theorem~3.2, Proposition~3.4]{paulin2015concentration}, for $\phi: \Omega^2 \to \R^+$,
\begin{equation*}
\begin{split}
\Var[\tilde{\bpi}]{\sum_{t=1}^{m - 1} \phi(Y_t)} &\leq \frac{4 m}{\pssg} \Var[\tilde{\bpi}]{\phi} \leq 8 m \tmix \Var[\tilde{\bpi}]{\phi}, \\
\end{split}
\end{equation*}
and from the fact that
\begin{equation*}
\begin{split}
\Var[\bpi]{\pred{X_t = i, X_{t+1} = j}} = \bpi(i)\mc(i,j)(1 - \bpi(i)\mc(i,j)),
\end{split}
\end{equation*}
we get the following control on the variance term:
\begin{equation*}
\begin{split}
\Var[\bpi]{N_{i j}} \leq 8 m \tmix \bpi(i)\mc(i,j)(1 - \bpi(i)\mc(i,j)).
\end{split}
\end{equation*}
\end{corollary}

\noindent As a consequence of Corollary~\ref{corollary:control-variance-empirical-transition-counts},
\begin{equation*}
\begin{split}
\mathbb{E}_{\bpi}\TV{\estQ - \Q} &\leq  2 \sqrt{\frac{\tmix + 1}{m - 1} \TV{\Q}_{1/2}}, \text{ where } \TV{\Q}_{1/2} \eqdef \left( \sum_{(i,j) \in \Omega^2} \sqrt{\Q(i,j) } \right)^2
\\
\end{split}
\end{equation*}
and
\begin{equation*}
\begin{split}
m \geq 64 \frac{\TV{\Q}_{1/2} \tmix}{ \eps^2} \implies \mathbb{E}_{\bpi}\TV{\estQ - \Q} \leq \eps/2.
\end{split}
\end{equation*}

\paragraph{Bounding the fluctuations around the expectation.}

The strategy is to show that the loss is $(4/(m-1))$-Lipschitz, and simply invoke McDiarmid's inequality for Markov chains \citep[Corollary~2.10]{paulin2015concentration}. \\\\
For $\x = (x_1, \dots, x_m) \in \R^m$, let $\x^{(k)} = (x_1, \dots, x_{k-1}, x_k', x_{k+1}, \dots, x_m)$, with $k \in [m]$,
\begin{equation*}
\begin{split}
&\abs{\TV{\estQ(\x) - \Q} - \TV{\estQ(\x') - \Q}} \\
&= \frac{1}{m-1} \abs{ \sum_{(i,j) \in \Omega^2} \left( \abs{n_{i j} - (m-1)\Q(i,j)} - \abs{n^{(k)}_{i j} - (m-1)\Q(i,j)} \right) } \\
&\leq \frac{1}{m-1}  \sum_{(i,j) \in \Omega^2} \abs{ \abs{n_{i j} - (m-1) \Q(i,j)} - \abs{n^{(k)}_{i j} - (m-1) \Q(i,j)} } \\
&\leq \frac{1}{m-1}  \sum_{(i,j) \in \Omega^2} \abs{ n_{i j} - n^{(k)}_{i j} },
\end{split}
\end{equation*}
where we successively invoked the forward and reverse triangle inequality. We then compute
\begin{equation*}
\begin{split}
\sum_{(i,j) \in \Omega^2} \abs{ n_{i j} - n^{(k)}_{i j} } = & \sum_{(i,j) \in \Omega^2} \bigg| \pred{x_k = i, x_{k+1} = j} - \pred{x_k^{(k)} = i, x_{k+1}^{(k)} = j} + \\
&  \pred{x_{k-1} = i,x_{k} = j} - \pred{x_{k-1}^{(k)} = i,x_k^{(k)} = j} \bigg| \\
\leq &\sum_{(i,j) \in \Omega^2} \bigg( \pred{x_k = i,x_{k+1} = j} + \pred{x_k' = i,x_{k+1} = j} + \\
& \pred{x_{k-1} = i,x_{k} = j} + \pred{x_{k-1} = i,x_k' = j} \bigg) \leq 4,
\end{split}
\end{equation*}
so that $\abs{\TV{\estQ(\x) - \Q} - \TV{\estQ(\x') - \Q}} \leq \frac{4}{m-1}$, and it is then a consequence of McDiarmid's inequality for Markov chains \citep[Corollary~2.10]{paulin2015concentration} that
\begin{equation*}
\begin{split}
\PR[\bpi]{\abs{ \TV{\estQ(\X) - \Q} - \mathbb{E}_{\bpi}\TV{\estQ(\X) - \Q} } > \eps/2} &\leq 2 \expo{-\frac{m \eps^2}{c \tmix }}, c \in \R_+\\
\end{split}
\end{equation*}
and $m \geq \frac{c \tmix}{\eps^2} \ln \left( \frac{2}{\delta} \right) \implies \PR[\bpi]{\abs{ \TV{\estQ - \Q} - \mathbf{\bpi} \TV{\estQ(\x) - \Q} } > \eps/2} \leq \delta$. 
Finally, we extend the study to non-stationary chains in a straightforward way as for the proof of Theorem~\ref{sec:proof-ub-finite}, with \citet[Proposition~3.10]{paulin2015concentration}, which yields the final theorem. \hfill\ensuremath{\square}

\subsection{Proof of Theorem~\ref{thm:learn-lb} (part 1): \texorpdfstring{$ \frac{d}{\eps^2 \pimin} $}{precision lower bound}}
Recall the definition of KL divergence between two distributions $(\boldsymbol{\nu}, \boldsymbol{\theta}) \in \Delta_\Omega^2$, such that $\boldsymbol{\nu} \ll \boldsymbol{\theta}$,
\begin{equation*}
\kl{\boldsymbol{\nu}}{\boldsymbol{\theta}} \eqdef \sum_{i \in \Omega} \boldsymbol{\nu}(i) \ln \frac{\boldsymbol{\nu}(i)}{\boldsymbol{\theta}(i)}.
\end{equation*}
Let $\eps \in (0, 1/32)$, and $\M_{d, \pssg, \pimin}$ be the collection of all $d$-state Markov chains whose stationary distribution is minorized by $\pimin$ and whose pseudo-spectral gap is at least $\pssg$. The quantity we wish to lower bound is the minimax risk for the estimation problem :
\begin{equation}
\label{eq:mmrisk}
\begin{split}
\mmrisk &= \inf_{\estmc} \sup_{\mc} \probm{\mc}{\nrm{\mc - \estmc}_\infty > \eps},
\end{split}
\end{equation}
where the $\inf$ is taken over all estimation procedures $\estmc: [d]^m \to \mathcal{M}_d, \X \mapsto \estmc(\X)$ and the $\sup$ over $\M_{d,\pssg,\pimin}$. Suppose for simplicity of the analysis that we consider Markov chains of $d+1$ states instead of $d$, and that $d$ is even. A slight modification of the proofs covers the odd case. We define the following class of Markov chains parametrized by a given distribution $\boldsymbol{p} \in \Delta_{d+1}$, where the conditional distribution defined at each state of the chain is always $\boldsymbol{p}$ with $p_{d+1} = p_\star$ and $p_k = \frac{1 - p_\star}{d}$ for $k \in [d]$, with $p_\star < \frac{1}{d+2}$, except for state $d+1$, where it is only required that it has a loop of probability $p_\star$ to itself.
\begin{equation}
\label{definition:G-p}
\Gclass_{\boldsymbol{p}} = \set{
\mc_{\etab}
= \begin{pmatrix}
 p_1 & \hdots & p_d & p_\star \\
  \vdots & \vdots & \vdots & \vdots\\
  p_1 & \hdots & p_d & p_\star \\
 \eta_1 & \hdots & \eta_d & p_\star \\
\end{pmatrix} :
\etab = (\eta_1, \dots, \eta_d, p_\star) \in \Delta_{d+1}
}.
\end{equation}
Remark: a family of Markov chains very similar to $\Gclass_{\boldsymbol{p}}$ was independently considered by \citet{haoorlitsky2018} for proving their lower bound.

\noindent It is easy to see that the stationary distribution $\bpi$ of an element of $\Gclass_{\boldsymbol{p}}$ indexed by $\etab$ is
\begin{equation*}
\pi_k = \frac{(1 - p_\star)^2}{d} + \eta_k p_\star, \text{ for } k \in [d], \qquad \pi_{d+1} = p_\star.
\end{equation*}
For $m \geq 4$, $\etab =(\eta_1, \dots, \eta_{d}, p_\star) \in \Delta_{d+1}$ and $(X_1, \dots, X_m) \sim (\mc_{\etab}, \boldsymbol{p})$, set 
$N_i = \sum_{t=1}^{m} \pred{X_t = i}$
the number of visits to the $i$th state. Focusing on the $(d+1)$th state, since
for $ i \in [d+1]$, we have $\mc_{\etab}(i, d+1) = p_\star$,
it is immediate that $N_{d+1} \sim \Binomial(m, p_\star)$.
Introduce the subset of Markov chains in $\Gclass_{\boldsymbol{p}}$ such that 
$$\etab(\boldsymbol{\sigma}) = \left( \frac{1 - p_\star + 16 \sigma_1 \eps}{d}, \frac{1 - p_\star - 16 \sigma_1 \eps}{d}, \dots, \frac{1 - p_\star + 16 \sigma_{\frac{d}{2}} \eps}{d}, \frac{1 - p_\star - 16 \sigma_{\frac{d}{2}} \eps}{d} , p_\star \right),$$ where $\boldsymbol{\sigma} = \left( \sigma_1, \dots, \sigma_{\frac{d}{2}} \right) \in \set{-1, 1}^{\frac{d}{2}}$. Also define $\mc_0$ with $\etab_0 = \left( \frac{1 - p_\star}{d}, \dots , \frac{1 - p_\star}{d} , p_\star \right)$. 
We start by showing that
for
any chain of this family,
$\pssg$ is bounded from below by a universal constant. 
The Dobrushin coefficient $\kappa$ [defined at \eqref{eq:dobr}]
verifies
$$\kappa(\mc_{\boldsymbol{\sigma}}) = \tv{\etab(\boldsymbol{\sigma}) - \etab_0} = 8 \eps \leq 1/2.$$
From the Bubley-Dyer path coupling method \citep{bubley1997path}, $\tmix \leq \frac{\ln{1/4}}{\ln{(1 - \kappa)}} \leq 2$, such that
combining with \citet[Proposition~3.4]{paulin2015concentration}, $\pssg \geq \frac{1}{2 \tmix} \geq \frac{1}{4}$.
A direct computation yields that for $\boldsymbol{\sigma} \neq \boldsymbol{\sigma}'$, $\nrm{\mc_{\boldsymbol{\sigma}} - \mc_{\boldsymbol{\sigma'}}}_1 = \frac{32 \eps}{d} d_H(\boldsymbol{\sigma}, \boldsymbol{\sigma}')$, where $d_H$ is the Hamming distance.
From the Varshamov-Gilbert lemma, we know that
there is a
$ \Sigma \subset \set{-1,1}^{d/2}$, $\abs{\Sigma} \geq 2^{d/16}$, such that
for $(\boldsymbol{\sigma}, \boldsymbol{\sigma}') \in \Sigma$ with $\boldsymbol{\sigma} \neq \boldsymbol{\sigma}'$,
we have
$d_H(\boldsymbol{\sigma},  \boldsymbol{\sigma}') \geq \frac{d}{16}$. Restricting our problem to this set $\Sigma$,
and finally noticing that for $ \boldsymbol{\sigma} \in \Sigma$ we have $\nrm{\mc_{\boldsymbol{\sigma}} - \mc_{0}}_1 = 16 \eps > 2 \eps$,
applying Tsybakov's method \citep[Theorem~2.5]{tsybakov2009introduction} to our problem, we obtain
\begin{equation*}\begin{split}
\mathcal{R}_m &\geq \frac{1}{2} \left( 1 - \cfrac{\frac{4}{2^{\frac{d}{16}}} \sum_{\boldsymbol{\sigma} \in \Sigma} \mathcal{D}_m}{\ln{2^{\frac{d}{16}}}} \right), \\
\end{split}\end{equation*}
where $\mathcal{D}_m$ denotes the KL divergence between the two distributions of words of length $m$
(see formal definition at Lemma~\ref{lemma:tensorization-kl}) from each of the Markov chains indexed by $\etab(\boldsymbol{\sigma})$ and $\etab_0$. Leveraging the chain rule for the KL divergence, and as by construction, the only discrepancy occurs when visiting the $(d+1)$th state, Lemma~\ref{lemma:tensorization-kl} shows the following tensorization property,
\begin{equation}
\label{eq:tensor}
\mathcal{D}_m \leq p_\star m \kl{\etab(\boldsymbol{\sigma})}{\etab_0},
\end{equation}
following up with a straightforward computation,
\begin{equation}
\label{eq:kl-eta}
\begin{split}
\kl{\etab(\boldsymbol{\sigma})}{\etab_0} &= \sum_{s \in \set{-1, +1}} \frac{d}{2} \left( \frac{1 - p_\star +  16 s \eps}{d} \right) \ln\left( \frac{\frac{1 - p_\star + 16 s \eps}{d}}{\frac{1 - p_\star}{d}} \right) \leq 128 \eps^2,
\end{split}\end{equation}
and finally combining \eqref{eq:mmrisk}, \eqref{eq:tensor} and \eqref{eq:kl-eta}, we get $\mathcal{R}_m \geq \frac{1}{2} \left( 1 - \cfrac{512 \eps^2 m p_\star}{\frac{d}{16} \ln{2}} \right)$. Further noticing that for the considered range of $\eps$ and for $p_\star < \frac{1}{d + 2}$, it is always the case that $\pimin = p_\star$, so that for $m \leq \frac{d (1 - 2 \delta) \ln {2}}{8192 \eps^2 \pimin}, \mmrisk \geq \delta$. \hfill\ensuremath{\square}

\begin{lemma}
\label{lemma:tensorization-kl}
For two Markov chains $\mc_{1}$ and $\mc_{2}$ of the class $\Gclass_{\boldsymbol{p}}$ defined at \eqref{definition:G-p} indexed respectively by $\etab_1$ and $\etab_2$, 
denote respectively by $\law_1(\X_1^m)$ and $\law_2(\X_2^m)$ the distributions of words of length $m$, and write for simplicity $\mathcal{D}_t = \kl{\law_1(\X_1^t)}{\law_2(\X_2^t)}$ the KL divergence between the processes up to time $t$. Then it is a fact that
\begin{equation*}
\mathcal{D}_m = (m-1) p_\star\kl{\etab_1}{\etab_2}.
\end{equation*}

\begin{proof}
From an application of the chain rule for the KL divergence, followed by the Markov property,
\begin{equation*}
\mathcal{D}_m = \mathcal{D}_{m-1} + \mathcal{E}_m, \\
\end{equation*}
where 
\begin{equation*}
\begin{split}
\mathcal{E}_m &\eqdef \E[\X_1^{m-1} \sim \law_1]{ \mathcal{D}_{m \gn m-1} }, \\
\mathcal{D}_{m \gn m-1} &\eqdef  \kl{\law_1(X_m \gn X_{m-1} )}{\law_2(X_m \gn X_{m-1} )}.
\end{split}
\end{equation*}
In the event where $X_{m-1} \neq d + 1$, $\mathcal{D}_{m \gn m-1} = 0$, such that from the law of total expectation,
\begin{equation}
\label{eq:expectation-term}
\mathcal{E}_m = \E[\X_1^{m-2} \sim \law_1]{ \E{ \mathcal{D}_{m \gn m-1}  \gn X_{m-1} = d + 1} \law_1( X_{m-1} = d + 1 \gn \X_1^{m-2}) }.
\end{equation}
From a second application of the Markov property, and by structural property of the chain,
\begin{equation}
\label{eq:using-structural-property}
\law_1( X_{m-1} = d + 1 \gn \X_1^{m-2}) = \law_1( X_{m-1} = d + 1 \gn X_{m-2}) = \mc_{1}(X_{m-2}, d + 1) = p_\star,
\end{equation}
while in the event where $X_{m-1} = d + 1$,
\begin{equation}
\label{eq:kl-conditioned-on-last-visited}
\mathcal{D}_{m \gn m-1} = \kl{\etab_1}{\etab_2}.
\end{equation}
Combining \eqref{eq:expectation-term} with \eqref{eq:using-structural-property} and \eqref{eq:kl-conditioned-on-last-visited} yields,
\begin{equation*}
\mathcal{E}_m = p_\star\kl{\etab_1}{\etab_2}.
\end{equation*}
From an inductive argument, and the base case $\mathcal{D}_1 = 0$,
\begin{equation*}
\mathcal{D}_m = (m-1) p_\star\kl{\etab_1}{\etab_2}.\\
\end{equation*}
\end{proof}
\end{lemma}

\subsection{Proof of Theorem~\ref{thm:learn-lb} (part 2): \texorpdfstring{$ \frac{d \ln {d}}{\pssg} $}{mixing lower bound}}
We treat $\eps \in (0, 1/8)$ and
$d=6k$, $k\ge2$ as fixed.
For $\eta\in(0,1/48)$
and $\btau\in\set{0,1}^{d/3}$,
define
the
block matrix
\beq
\mc_{\eta,\btau}
  = 
\begin{pmatrix}
\boldsymbol{C}_{\eta} & \boldsymbol{R}_{\btau} \\
\boldsymbol{R}_{\btau} \trn & \boldsymbol{L}_{\btau}
\end{pmatrix}
\eeq
, where
$\boldsymbol{C}_{\eta} \in \R^{d/3\times d/3}$,
$  \boldsymbol{L}_{\btau}  \in\R^{2d/3\times2d/3}$,
and
$\boldsymbol{R}_{\btau} \in\R^{d/3\times 2d/3}$
are given by
$$
  \boldsymbol{L}_{\btau} = \frac{1}{8} \diag\left( 7 - 4 \tau_1 \eps, 7 + 4 \tau_1 \eps, \dots, 7 - 4 \tau_{d/3} \eps, 7 + 4 \tau_{d/3} \eps  \right)
,$$
\beq
\boldsymbol{C}_\eta
  =
\begin{pmatrix}
\frac{3}{4} - \eta & \frac{\eta}{d/3 - 1} & \hdots & \frac{\eta}{d/3 - 1} \\
\frac{\eta}{d/3 - 1} & \frac{3}{4} - \eta & \ddots & \vdots \\
\vdots & \ddots & \ddots & \frac{\eta}{d/3 - 1} \\
\frac{\eta}{d/3 - 1} & \hdots & \frac{\eta}{d/3 - 1} & \frac{3}{4} - \eta \\
\end{pmatrix}
,
\eeq
\beq
\boldsymbol{R}_{\btau} = \frac{1}{8}
\begin{pmatrix}
1 + 4 \tau_1 \eps & 1 - 4 \tau_1 \eps  & 0 & \hdots & \hdots & \hdots & 0 \\
0 & 0 & 1 + 4 \tau_2 \eps & 1- 4 \tau_2 \eps  & 0 & \hdots & 0 \\
\vdots & \vdots & \vdots & \vdots & \vdots & \vdots & \vdots \\
0 & \hdots & \hdots & \hdots & 0 & 1 + 4 \tau_{d/3} \eps & 1- 4 \tau_{d/3} \eps \\
\end{pmatrix}
.
\eeq
Holding $\eta$ fixed, define the collection
\beqn
\label{eq:Heta}
\H_\eta=\set{\mc_{\eta,\btau}:\btau\in\set{0,1}^{d/3}}
\eeqn
of Markov matrices.
Denote by $\mc_{\eta,\zero}\in\H_\eta$
the element corresponding to $\btau=\zero$.
Note that every $\mc\in\H_{\eta}$
is
ergodic
and
reversible,
and its unique stationary distribution is uniform. 
A graphical illustration
of this class of
Markov chains is provided in Figure~\ref{fig:contracting-mc-topology};
in particular,
every
$\mc\in\H_{\eta}$
consists of an ``inner clique''
(i.e., the states indexed by $\set{1,\ldots,d/3}$)
and ``outer rim''
(i.e., the states indexed by $\set{d/3+1,\ldots,d}$).

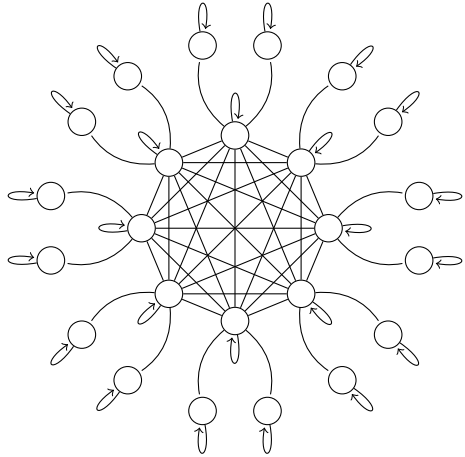
\begin{figure}[H]
\begin{center}
\begin{tikzpicture}[auto]
\small

\def \d {8}
\def \gap {10}
\def \radius {35pt}

\foreach \c in {1,...,\d}{

	\pgfmathparse{(\c - 1) * 360 / \d + 90}
	\node[draw,circle] (\c) at (\pgfmathresult: \radius) {};
	\pgfmathsetmacro\angleleft{(\c - 1) * 360 / \d + 90 - \gap}
	\pgfmathsetmacro\angleright{(\c - 1) * 360 / \d + 90 + \gap}
	\pgfmathtruncatemacro\idleft{\d + \c}
	\pgfmathtruncatemacro\idright{\d + \d + \c}
	\node[draw,circle] (\idleft) at (\angleleft: 70pt) {};
	\node[draw,circle] (\idright) at (\angleright: 70pt) {};

	\draw[every loop]
			(\c) edge[-, bend right, swap] node {} (\idleft)
			(\c) edge[-, bend left] node {} (\idright)
			(\idleft) edge[-,in=\angleleft,out=\angleright, loop] node {} (\idleft)
			(\idright) edge[-,in=\angleleft,out=\angleright, loop] node {} (\idright)
			(\c) edge[-,in=\angleleft,out=\angleright, loop] node {} (\c)
			;

}    

\foreach \c in {1,...,\d}{
	\foreach \e in {\c,...,\d}{
		    \ifthenelse{\c = \e}{}{\draw[] (\c) edge[-] node {} (\e);}
	}
}
\end{tikzpicture}
\end{center}
\caption{Generic topology of the $\H_{\eta}$ Markov chain class:
  every chain consists of an ``inner clique'' and an ``outer rim''.}
\label{fig:contracting-mc-topology}
\end{figure}

Lemma~\ref{lemma:control-spectral-gap} in the Appendix
establishes a key property of
the elements of $\H_\eta$: each $\mc$ in this class satisfies
$$\eta/4 \leq \pssg \leq \eta.$$
Suppose that $\X=(X_1,\ldots, X_m) \sim(\mc_{\etab},\bpi)$,
where $\mc\in\H_{\eta}$ and $\bpi$ is uniform.
Define
the random variable $T\cliq$,
to be the first time all of the states in the inner clique were visited,
\beqn
\label{eq:Tcliq}
T\cliq =
\inf\set{t\ge1:
  \abs{\set{X_1,\ldots,X_t}\cap[d/3]}
  =d/3},
\eeqn
Lemma~\ref{lemma:half-cover} in the Appendix
 gives a lower estimate on this quantity:
$$m \leq \frac{d}{20 \eta} \ln{\left( \frac{d}{3} \right)} \implies \PR{\tall > m} \geq \frac{1}{5} .$$
Let $\M_{d,\pssg,\pimin}$
be the collection of all $d$-state Markov chains whose stationary distribution is minorized by
$\pimin$ and whose pseudo-spectral gap is at least $\pssg$.
Writing $\X = (X_1, \dots, X_m)$, recall that the quantity we wish to lower bound is the minimax risk for the statistical estimation problem
(it will be convenient to write $\eps/2$ instead of $\eps$, which only affects the constants):
\begin{equation*}
\begin{split}
\mmrisk &= \inf_{\estmc} \sup_{\mc} \probm{\mc}{\nrm{\mc - \estmc}_\infty > \frac{\eps}{2}},
\end{split}
\end{equation*}
where the $\inf$ is taken over all
estimation procedures $\estmc: \Omega^m \to \M_d, \X \mapsto \estmc(\X)$,
and the $\sup$ over $\M_{d,\pssg,\pimin}$.
We employ the
general reduction scheme
of
\citet[Chapter~2.2]{tsybakov2009introduction}.
The first step is to restrict the $\sup$ to
the finite subset $\H_{\eta} \subsetneq \M_{d,\pssg,\pimin}$.

\begin{equation*}
\begin{split}
  \mmrisk &\geq \inf_{\estmc} \max_{\btau} \probm{\mc_{\eta, \btau}}{\nrm{\mc_{\eta, \btau} - \estmc}_\infty > \frac{\eps}{2}}
  .
\end{split}
\end{equation*}
Define
$\tall$
as in
(\ref{eq:Tcliq}).
Then
\begin{equation*}
\begin{split}
  \mmrisk &\geq \inf_{\estmc} \max_{\btau} \probm{\mc_{\eta, \btau}}{\nrm{\mc_{\eta, \btau} - \estmc}_\infty > \eps \gn \tall > m}
  \probm{\mc_{\eta, \btau}}{\tall > m}
\end{split}
\end{equation*}
and
Lemma~\ref{lemma:half-cover} implies that
for $m < \frac{d}{20 \eta} \ln{\left(\frac{d}{3}\right)}$,
\begin{equation*}
\begin{split}
  \mmrisk &\geq \frac{1}{5} \inf_{\estmc} \max_{\btau}
  \probm{\mc_{\eta, \btau}}{\nrm{\mc_{\eta, \btau} - \estmc}_\infty > \eps \gn \tall > m}
  .
\end{split}
\end{equation*}
Observe that
all $\btau\neq\btau'\in
\set{0, 1}^{d/3}
$
verify
$\nrm{\mc_{\eta, \btau} - \mc_{\eta, \btau'}}_\infty = \eps$.
For any
estimate
$\estmc$,
define
$$\btau^\star = \argmin_{\btau} \nrm{\estmc - \mc_{\eta, \btau}}_\infty.$$
Then for $\btau\neq\btau^\star$, we have
\begin{equation*}
\begin{split}
\eps  &= \nrm{ \mc_{\eta, \btau} - \mc_{\eta, \btau^\star} }_\infty \leq \nrm{\mc_{\eta, \btau} - \estmc }_\infty + \nrm{\estmc - \mc_{\eta, \btau^\star}}_\infty \leq 2  \nrm{\mc_{\eta, \btau} - \estmc }_\infty,
\end{split}
\end{equation*}
whence
$\set{ \btau^\star \neq \btau } \subset \set{ \nrm{ \mc_{\eta, \btau} - \estmc }_\infty > \eps/2 }$
and
\begin{equation*}
\begin{split}
  \mmrisk &\geq
  \frac{1}{5} \inf_{\estmc} \max_{\btau} \probm{\mc_{\eta, \btau}}{ \btau^\star \neq \btau \gn \tall > m} \\
	&=
  \frac{1}{5} \inf_{
\estbtau: \X \mapsto \set{0, 1}^{d/3}
  } \sup_{\btau}
   \probm{\mc_{\eta, \btau}}{ \estbtau \neq \btau \gn \tall > m}
  .
\end{split}
\end{equation*}
Since $\tall>m$
implies that
$N_{i}=0$ for some $i \in [d/3]$,
\begin{equation*}
\begin{split}
  \mmrisk &\geq \frac{1}{5} \inf_{\estbtau} \sup_{\btau} \probm{\mc_{\eta, \btau}}
          { \esttau_{i} \neq \tau_{i} \gn N_{i} = 0}.
\end{split}
\end{equation*}
There are as many
$\mc\in\H_{\eta}$
with
$\tau_{i} = 0$
as those with
$\tau_{i} = 1$,
so if $\mc$ is drawn uniformly at random
and state $i$ has not been visited,
one can do no better than to make a random choice of $\esttau_{i}$
(where $\estbtau$ determines $\estmc$).
More formally,
writing
$\btau^{(i)}= (\tau_1, \dots, \tau_{i-1}, \tau_{i+1}, \dots, \tau_{d/3})
\in\set{0,1}^{d/3-1}
$,
the $\btau$ vector without its $i$th coordinate,
we can
employ an
Assouad-type of decomposition \citep{assouad1983deux, yu1997assouad}:
\begin{equation*}
\begin{split}
  \mmrisk &\geq \frac{1}{5} \inf_{\estbtau} 2^{1-d/3} \sum_{\btau^{(i)} \in \set{0, 1}^{d/3 - 1}} \bigg[ \frac{1}{2} \probm{\tau_i = 0}{ \esttau_i \neq \tau_i \gn N_i = 0}
    + \frac{1}{2} \probm{\tau_i = 1}{ \esttau_i \neq \tau_i \gn N_i = 0} \bigg] \\
  &= \frac{2^{1-d/3}}{10} \sum_{\btau^{(i)} \in \set{0, 1}^{d/3 - 1}} \inf_{\estbtau} \left[ \probm{\tau_i = 0}{ \esttau_i = 1 \gn N_i = 0} +
    \probm{\tau_i = 1}{ \esttau_i = 0 \gn N_i = 0} \right] \\
&= \frac{2^{1-d/3}}{10} \sum_{\btau^{(i)} \in \set{0, 1}^{d/3 - 1}} \left[ 1 - \tv{ \probm{\tau_i = 0}{ \X = \cdot \gn N_i = 0} +  \probm{\tau_i = 1}{ \X = \cdot \gn N_i = 0} } \right] \\
&= \frac{1}{10}.
\end{split}
\end{equation*}
Combined with
Lemma~\ref{lemma:control-spectral-gap},
and inclusion of events,
this implies lower bound of
$\frac{d}{\pssg} \ln{d}$
for the estimation problem,
which is tight for the case
$\pimin = \frac{1}{d}$.\hfill\ensuremath{\square}

\begin{remark}
  \label{rem:HSK}
  Let us compare construction $\H_\eta$ to the family of Markov chains
  employed in the lower bound of \citet{hsu2019}:
\begin{equation*}
\begin{split}
\mc(i,j) = 
\begin{cases}
1 - \eta_i,& i = j \\
\frac{\eta_i}{d-1},& \text{else}
\end{cases}
,
\end{split}
\end{equation*}
where $\eta_i \in \{\eta, \eta'\}$ with $\eta' \approx \eta/2$.
For our lower bound, $\H'_\eta$ has to be a $\eps$-separated set under $\nrm{\cdot}_\infty$.
In the construction of \citeauthor{hsu2019}, the spectral gap $\sg$ and the separation distance $\eps$ are coupled,
and using their family of Markov chains would lead to a lower bound of order $d/\sg\approx d/\eps$,
which is inferior to
$\frac{d}{\eps^2 \pimin}$.
The free parameter $\eta$ was
key to our
construction, which enabled us to decouple $\sg$ from $\eps$.
\end{remark}

\begin{lemma}
\label{lemma:control-spectral-gap}
Let $\eps \in (0, 1/32)$ and $\eta \in (0, 1/48)$. For all $\mc \in \H_{\eta}$
[defined in (\ref{eq:Heta})],
we have
  $\eta/4 \leq \asg \leq \pssg \leq \eta$.
\end{lemma}
\begin{proof}
  We focus our proof on the absolute spectral gap, and will later show that the pseudo spectral gap is of the same order for our class of Markov matrices.
  A lower bound for $\asg$ of the unperturbed chain $\mc_{\eta,\zero}$, 
	is given by Lemma~\ref{lemma:exact-eigenvalues}.
We now show how to extend to general $\btau$ with comparison techniques.
It is well known that (see i.e. \citet[Lemma~13.7]{levin2009markov}) that for a reversible chain $\mc$,
\begin{equation}
\label{eq:definition-sg-variational}
\sg(\mc) = \min_{\substack{f: [d] \to \R \\ f \perp_{\bpi} \unit, \nrm{f}_2 = 1 }} \mathcal{E}_{\mc}(f)
\end{equation}
where
\begin{equation*}
\mathcal{E}_{\mc}(f) \eqdef \frac{1}{2} \sum_{(i,j) \in [d]^2} (f(i) - f(j))^2 \bpi(i) \mc(i, j)
\end{equation*}
is the Dirichlet form associated to $\mc$ with stationary distribution $\bpi$. We now use this variational definition to control the spectral gap of the perturbed chains $\mc_{\eta, \btau}$ in terms of the one of $\mc_{\eta, \zero}$, relying on the fact that for both chains, the stationary distribution is uniform. Comparing transition matrices,
\begin{equation*}
\begin{split}
\mathcal{E}_{\mc_{\eta, \btau}}(f) &= \frac{1}{2} \sum_{(i,j) \in [d]^2} (f(i) - f(j))^2 \bpi(i) \mc_{\eta, \btau}(i, j) \\
&\geq \frac{1}{2} \sum_{(i,j) \in [d]^2} (f(i) - f(j))^2 (1/d) (1 - 4 \eps) \mc_{\eta, \zero}(i, j) \\
&= (1 - 4 \eps) \mathcal{E}_{\mc_{\eta, \zero}}(f),
\end{split}
\end{equation*}
and by the definition at \eqref{eq:definition-sg-variational}, $\sg(\mc_{\eta, \btau}) \geq (1 - 4\eps) \sg(\mc_{\eta, \zero})$.

\paragraph{Extension to $\pssg$.}

Now note that for a
symmetric and lazy
$\mc$,
$\bpi$ is the uniform
distribution,
$\mc \rev = \mc\trn = \mc$,
and
$\pssg = \max_{k \geq 1} \set{ \frac{\sg(\mc^{2k})}{k}}$.
Denoting by $1 = \lambda_1 > \lambda_2 \geq \dots \geq \lambda_d$ the eigenvalues of $\mc$,
we have that
for all $i \in [d]$ and $k \geq 1$,
$\lambda_i^{2k}$ is an eigenvalue for $\mc^{2k}$,
and furthermore
$1 = \lambda_1^{2k} > \lambda_2^{2k} \geq \dots \geq \lambda_d^{2k}$.
Then
\beq
\pssg = \max_{k \geq 1} { \frac{1 - \lambda_2^{2k}}{k}}= 1 - \lambda_2^{2}
\eeq
--- that is, the maximum is achieved at $k=1$.
Indeed,
$1 - \lambda_2^{2k} =  (1 - \lambda_2^2)\left( \sum_{i = 0}^{k-1} \lambda_2^{2i} \right)$
and the latter sum is at most $k$ since 
$\lambda_2^2 < \lambda_2 <1$.
As a result, $\pssg = 1 - \lambda_2^2 = 1 - (1 - \sg)^2 = \sg(2 -  \sg)$ and
\begin{equation*}
\begin{split}
\sg \leq \pssg \leq 2\sg,
\end{split}
\end{equation*}
which completes the proof.
\end{proof}
\begin{lemma}[Cover time]
  \label{lemma:half-cover}
  For
  $\mc \in \H_{\eta}
  $
  [defined in (\ref{eq:Heta})],
  the random variable
$T\cliq$
  [defined in (\ref{eq:Tcliq})]
satisfies
\begin{equation*}
\begin{split}
m \leq \frac{d}{20 \eta} \ln{\left( \frac{d}{3} \right)} &\implies \PR{\tall > m} \geq \frac{1}{5} \\
\end{split}
\end{equation*}
\end{lemma}
\begin{proof}

  Let $\mc\in\H_\eta$ and $\mc_I \in \M_{d/3}$ be
  such that $\mc_I$ consists only in the inner clique of $\mc$,
  and each outer rim state got absorbed into its unique inner clique neighbor:
  $$\mc_I = \begin{pmatrix}
1 - \eta & \frac{\eta}{d/3 - 1} & \hdots & \frac{\eta}{d/3 - 1} \\
\frac{\eta}{d/3 - 1} & 1 - \eta & \ddots & \vdots \\
\vdots & \ddots & \ddots & \frac{\eta}{d/3 - 1} \\
\frac{\eta}{d/3 - 1} & \hdots & \frac{\eta}{d/3 - 1} & 1 - \eta \\
  \end{pmatrix}.$$
  By construction, it is clear that $\tall$ is
almost surely
greater than the cover time of $\mc_I$.
The latter
corresponds to a generalized coupon collection time $U = 1 + \sum_{i=1}^{d/3 - 1}U_i$ where $U_i$
is
the time increment between the $i$th and the $(i+1)$th
unique visited state.
Formally,
if $\X$ is a random walk according to $\mc_I$ (started from any state),
then
$U_{1} = \min \{ t > 1 : X_t \neq X_1 \}$
and
for $i>1$,
\begin{equation*}
  U_{i} = \min \{ t >
  1
  :
  X_t \notin \{ X_1, \dots, X_{U_{i-1}} \} \} - U_{i-1}.
\end{equation*}
The random variables $\thit{1}, \thit{2}, \dots, \thit{d/3 - 1}$
are independent and 
$$\thit{i} \sim \Geometric \left( \eta - \cfrac{(i-1)\eta}{d/3} \right),$$ 
whence
\begin{equation*}
\begin{split}
  \E{\thit{i}} = \frac{d/3}{\eta(d/3 - i + 1)},
\qquad
  \Var{\thit{i}} = \left(1 - \left( \eta - \cfrac{(i-1)\eta}{d/3} \right)\right)\left( \eta - \cfrac{(i-1)\eta}{d/3} \right)^{-2}
\end{split}
\end{equation*}
and
\begin{equation*}
\begin{split}
  \E{U\cover} \geq 1 + \cfrac{d/3}{\eta} \sigma_{d/3 - 1}, \qquad
  \Var{U\cover} \leq \cfrac{(d/3 - 1)^2}{\eta^2} \cfrac{\pi^2}{6} \\
\qquad
\end{split}
\end{equation*}
where $\sigma_{d} = \sum_{i = 1}^{d} \frac{1}{i}$, and $\pi = 3.1416\dots$. Invoking
the Paley-Zygmund inequality with
$\theta = 1 - \frac{2\sqrt{2/3}}{\sigma_{d/3-1}}$
we have
\begin{equation*}
\begin{split}
  \PR{U\cover > \theta \E{U\cover}} &\geq
  \paren{1 + \cfrac{\Var{U\cover}}{(1 - \theta)^2 (\E{U\cover})^2}}\inv \geq \frac{1}{5}.
\end{split}
\end{equation*}
Further,
$\sigma_{d/3 - 1} \geq \sigma_3 = 11/6$
implies
$$\theta \E{U\cover} \geq \frac{3}{20} \cdot \frac{d/3}{\eta} \sigma_{d/3 - 1} \geq \frac{d}{20 \eta} \ln{\left(\frac{d}{3}\right)},$$
and thus for
$m \leq \frac{d}{20 \eta} \ln{\left(\frac{d}{3}\right)}$,
we have
$ \PR{\tall > m} \geq \frac{1}{5}$.

\end{proof}
\begin{lemma}[Spectrum of $\mc_{\eta,\zero} \in \H_{\eta}$]
\label{lemma:exact-eigenvalues}

Let $d=6k$, $k\ge2$ and $0 < \eta < 1/48$, and write $c_d \eqdef d/(d-3)$.
The spectrum of $\mc_{\eta,\zero}$ is $$\Spec\left(\mc_{\eta,\zero} \right) = \set{\lambda_1, \lambda_{+}, \overline{\lambda}, \underline{\lambda}, \lambda_{-}}$$
where $ \lambda_1= 1$ (mult. $1$), $\lambda_{\pm} \eqdef \frac{1}{16} \left( 13 - 8 \eta c_d \pm \sqrt{(3 + 8 \eta c_d /3)^2 + 512 \eta^2 c_d^2 /9} \right)$ (each mult. $d/3-1$) $, \underline{\lambda} = 5/8$ (mult. $1$) $, \overline{\lambda}= 7/8$ (mult. $d/3$).
\\\\ Moreover, $\lambda_\star \eqdef \max_{\lambda \in \Spec(\mc_{\eta,\zero})} \set{\abs{\lambda}, \lambda \neq 1} = \lambda_{+}$, and
$ \eta/4 \leq \asg \leq \eta/2$.

\end{lemma}

\begin{proof}  

By definition, and writing $c_d \eqdef d/(d-3)$, $\mc_{\eta, \zero}
  = 
\begin{psmallmatrix}
\boldsymbol{C}_\eta & \boldsymbol{R}_{\zero} \\
\boldsymbol{R}_{\zero} \trn & \boldsymbol{L}_{\zero}
\end{psmallmatrix}
,$
where $\boldsymbol{L}_{\zero} = \frac{7}{8} \identity$,
$$
  \boldsymbol{C}_\eta
  = \frac{\eta}{d/3 - 1} \unit \trn \cdot \unit  - \left( \eta c_d - \frac{3}{4} \right) \identity \text{ and } \boldsymbol{R}_{\zero} = \frac{1}{8}
\begin{pmatrix}
1 & 1 & 0 & \hdots & \hdots & \hdots & 0 \\
0 & 0 & 1  & 1  & 0 & \hdots & 0 \\
\vdots & \vdots & \vdots & \vdots & \vdots & \vdots & \vdots \\
0 & \hdots & \hdots & \hdots & 0 & 1 & 1  \\
\end{pmatrix}
.
$$
As $\mc_{\eta, \zero}$ is a symmetric matrix, its spectrum is real.
Let $\lambda \in \R$, and suppose first that $\lambda \neq 7/8$.
In this case, $\abs{ \boldsymbol{L}_{\zero} - \lambda \identity } \neq 0$, 
and leveraging the block-structure of the matrix, 
it is a classical result (see for example \citet{silvester2000}) that
$$\abs{\mc_{\eta, \zero} - \lambda \identity} = \abs{\boldsymbol{L}_{\zero} - \lambda \identity} \cdot \abs{\boldsymbol{C}_\eta - \lambda \identity - \boldsymbol{R}_{\zero}(\boldsymbol{L}_{\zero} - \lambda \identity)^{-1} \boldsymbol{R}_{\zero} \trn }.$$
A direct computation shows that %
\begin{equation*}
\begin{split}
\boldsymbol{R}_{\zero}(\boldsymbol{L}_{\zero} - \lambda \identity)^{-1} \boldsymbol{R}_{\zero} \trn &= \frac{1}{4(7 - 8\lambda)} \identity, \\
\end{split}
\end{equation*}
such that,
\begin{equation*}
\begin{split}
\boldsymbol{C}_\eta - \lambda \identity - \boldsymbol{R}_{\zero}(\boldsymbol{L}_{\zero} - \lambda \identity)^{-1} \boldsymbol{R}_{\zero} \trn
= \frac{\eta}{d/3 - 1} \unit \trn \cdot \unit - \left( \eta c_d - \frac{3}{4} + \lambda + \frac{1}{4(7 - 8\lambda)} \right) \identity. \\ 
\end{split}
\end{equation*}
This implies that $\abs{\boldsymbol{C}_\eta - \lambda \identity - \boldsymbol{R}_{\zero}(\boldsymbol{L}_{\zero} - \lambda \identity)^{-1} \boldsymbol{R}_{\zero} \trn} = 0$ if and only if,
$\eta c_d - \frac{3}{4} + \lambda + \frac{1}{4(7 - 8\lambda)} \in \Spec(\eta/(d/3 - 1) \unit \trn \cdot \unit ) = \set{0, \eta c_d}$ where $0$ has multiplicity $d/3 - 1$ and $\eta c_d$ has multiplicity $1$.
Let $\xi \geq 0$, then solutions for the equation $\xi + \lambda + \frac{1}{4(7 - 8 \lambda)} = \frac{3}{4}$ are given by
$$\lambda_{\pm}(\xi) = \frac{13 - 8 \xi \pm \sqrt{(3 + 8 \xi /3)^2 + 512 \xi^2 /9}}{16}.$$
Setting $\xi = 0$ yields that $\lambda_1 = 1$ and $\underline{\lambda} = 5/8$ are eigenvalues both with multiplicity $1$,
while setting $\xi = \eta c_d$ yields that
$$\lambda_{\pm} = \frac{13 - 8 \eta c_d \pm \sqrt{(3 + 8 \eta c_d /3)^2 + 512 \eta^2 c_d^2 /9}}{16},$$
are both eigenvalues with multiplicity $d/3 - 1$.
As the characteristic polynomial of $\mc_{\eta, \zero}$ has degree $d$, 
a natural consequence is that $\overline{\lambda} = 7/8$
is another eigenvalue with multiplicity $d/3$.
It remains to order $\lambda_1, \lambda_{-}, \lambda_{+}, \underline{\lambda}, \overline{\lambda}$.
Since $\mc_{\eta, \zero}$ is lazy, all eigenvalues are positive.
Trivially, $\lambda_1$ is the largest eigenvalue, $\overline{\lambda} > \underline{\lambda}$ and always $\lambda_{-} \leq \lambda_{+}$.
Additionally, $512 \eta^2 c_d^2 /9 \geq 0$ implies that $\lambda_{+} \geq 1 - \eta c_d/3 \geq 1 - \eta/2$ for the considered range of $d$, which is in turn larger than $7/8$ for $\eta \leq 1/4$. 
As a result $\lambda_\star = \lambda_{+}$ and $\asg \leq \eta/2$.
Furthermore, as one can write $\lambda_{+} = \frac{1}{16} \left( 13 - 8 \eta c_d \pm \sqrt{(3 + 4 \eta c_d /3)^2 - 8 c_d \eta(1 - 8c_d \eta)} \right)$,
and since $1 - 8c_d \eta \geq 0$, $\lambda_{+} \leq 1 - \eta/4$, whence $\asg \geq \eta/4$.

\end{proof}

\section*{Acknowledgments}
We are thankful to John Lafferty for bringing this problem to our attention and numerous insightful conversations.
We also thank the anonymous referees, who made valuable comments and suggestions, including
shaving off a logarithmic factor in
Theorem~\ref{thm:learn-ub}
and
the explicit computation of the eigenvalues
in Lemma~\ref{lemma:exact-eigenvalues}.
This research was partially supported by
the Israel Science Foundation
(grant No. 755/15),
Paypal and IBM.

\bibliography{bibliography}
\bibliographystyle{abbrvnat}

\end{document}